\documentclass[aps,floats,twocolumn,epsf,prb,showpacs]{revtex4-2}

\usepackage{latexsym,amsmath,amssymb,amsthm,bm,graphicx,amsfonts,booktabs}

\usepackage[pagewise]{lineno}
\usepackage[plain]{algorithm}
\usepackage[noend]{algpseudocode}
\usepackage{bm}
\usepackage{longtable}

\makeatletter\def\Dated@name{}\makeatother

\newtheorem{theorem}{Theorem}

\newcommand{\norm}[1]{\left\lVert#1\right\rVert}
\def\BState{\State\hskip-\ALG@thistlm}

\DeclareMathOperator{\interior}{int}
\DeclareMathOperator{\rank}{rank}

\begin{document}

\title{Evolutionary Data Theory: On the Similarities between Data Problems and Evolutionary Games}
\author{P. Wissgott$^{1}$}
\affiliation{$^1$ danube.ai solutions gmbh, 1040 Vienna, Austria\\
}

\begin{abstract}
Applying the concepts and formalism from Evolutionary Game Theory to the data regime, the fundamental paradigms of Evolutionary Data Theory are introduced. It is shown that essential definitions and results such as replicator equations, evolutionary strategies, the Bishop-Cannings theorem and the analogy to Lotka-Volterra systems can be mapped to the data interpretation. Understanding data in matrix form as evolutionary entities, input data is mapped to genes and organisms. Steered by genetic fitness and two evolutionary strategies, Dominant-Balanced and Altruistic-Selfish, data records and features conduct an evolutionary game. It is shown that this evolutionary interpretation remains universally meaningful, by proving convergence to a unique rest point, where all data features persist in the population. A basic example of multi-objective optimization is shown as well as a related distribution problem and machine learning applications.
\end{abstract}

\maketitle

\section{Introduction}
\label{sec:intro}
In the history of humanity, data have played a pivotal role to leverage innovation. From the geometric considerations building the pyramids to calculating the planetary orbits - those to understood data could investigate and plan the future.

Data can have many forms and shapes. However, for them to be usable to answer a question or solve a problem, data have to be in some way comparable with each other. Throughout history, and with major scientific milestones, the amount and complexity of available data have grown, while methods to compare data have become more evolved.

Today, the ways to investigate data are manifold: from algebraic approaches over statistics to the virtues of machine learning - never before have the choice of available methods been larger. Furthermore, the analytical and mathematical paradigms developed while investigating data problems have found their way in many, diverse fields of applications.

One of the most surprising, scientific 'cross-overs' of the last decades, is Evolutionary Game Theory~(EGT), which applies the formalisms of Game Theory to biological systems. With EGT, many types of animal behavior could be understood for the first time, e.g. the sex ratio in populations or the why there is an asymmetry between male and female zygotes in most populations~\cite{MaynardSmith1981}.

In this paper we suggest yet another disciplinary 'cross-over': what if we apply the evolutionary concepts of EGT to data problems? Indeed, many systems described by data have an underlying logic and dynamics quite similar to evolutionary means. For example, the concept of biological kinship between animals can be directly 'ported' to the data regime, where it represents a metric to describe similarity between data structures~\cite{daniilidis2025, wissgott2025}.

Why is this approach useful? Evolutionary Game Theory offers a rich tool set of theorems and derivations describing for example convergence, stability and persistence~\cite{MaynardSmith1981, HofbauerSigmund1998, Hofbauer2011}. Bringing these fundamental results to the world of data, gives us a head start in the investigation. 

We call this new theory, applying concepts of EGT to data, Evolutionary Data Theory~(EDT). Conducting this mapping, we will see below, systematic universality is one of the primary concerns. While in biology, the exact models may change switching from on system to the other, in the data regime we require our fundamental models to be universal. Consequently, our methodical approach should not impose unnecessary restriction on the input data. Apart from introductory purposes, proving this universality for EDT is the main focus of this paper.

Structurally, we will begin our derivation by revisiting major concepts and results from EGT in Sec.~\ref{sec:EGT}. In particular the fundamental definitions of replicator equations and the mapping to Lotka-Volterra equations~\ref{sec:lotkavolterraEGT} can be found here.

Next, we will introduce the fundamental entities, genes and organism, as well as the work flow of genetic conversion in Sec.~\ref{sec:formalism}. After a short motivation of our approach~\ref{sec:motivation}, we will formulate the replicator equations of EDT in Sec.~\ref{sec:replicatorequationsEDT}. Similar to EGT, to apply replicator equations in EDT, requires the definition of evolutionary strategies~(ES), which is shown in Sec.\ref{sec:evolutionarystrategies} followed by some general rules for ES in Sec.\ref{sec:rulesforevolutionarystrategies}. Note that we again use the 'Dominant-Balanced'~(DomBal) and 'Altruistic-Selfish'~(AltSel) combinations as first introduced in Wissgott~\cite{wissgott2025}. However, in this paper, we provide a much more detailed mathematical foundation of the ES~(in particular, for AltSel, we will derive major theorems for the first time).

As another primary novelty of this paper, we formulate the Bishop-Cannings theorem of EDT in Sec.~\ref{sec:bishopcannings}. How can we obtain evolutionary payoff matrices from input data? This is demonstrated in detail in Sec.~\ref{sec:payoffEDT}. In Sec.~\ref{sec:dynamicsandstability}, we show the convergence to a stable rest point for the ES DomBal and AltSel. Though the first proof for DomBal can be found in Daniilidis \emph{et al}, we provide an alternative proof using the equivalency to Lotka-Volterra systems. For the ES AltSel, we provide the first convergence proof to our knowledge.

Another central property in EDT is the persistence of all initially available data features or genes, i.e. conditions for the input data such that no feature vanishes throughout the dynamics. As an original contribution, we prove the persistence for AltSel in Sec.~\ref{sec:peristence}.

In Sec.~\ref{sec:applications and examples} we turn our attention to applications of EDT. First, we introduce our example data, a supermarket chain aiming to distribute a delivery. We show in Sec.~\ref{sec:optimization} how EDT can be applied to discrete multi-objective optimization~(see also Daniilidis \emph{et al}). We also visualize the iterative dynamics of DomBal and AltSel for this example. Next, we investigate the optimal distribution of the delivery among the available stores in Sec.~\ref{sec:distributionproblems}. Towards the end of this paper in Sec.~\ref{sec:machinelearning}, we pursue a methodological glimpse on ways to apply training similar to machine learning in EDT.

\section{A short reprise of Evolutionary Game Theory}
\label{sec:EGT}

Starting with Maynard-Smith's paper \emph{The Logic of Animal Conflict}~\cite{MaynardSmith1973} in 1973, Evolutionary Game Theory quickly became a useful method in the analysis of the dynamics of biological contests. The concept of evolutionary stability brought surprising explanation to animal behavior, sexual reproduction and the dynamics of the underlying genetic encoding.

Let us proceed, by introducing some fundamental definitions and formalisms of~EGT~(the part below is mainly following Hofbauer \emph{et al}~\cite{HofbauerSigmund1998}). Let us assume a population of animals is divided into $m$ types $E_1$ to $E_m$ with frequencies $\gamma_1$ to $\gamma_m$ where
\begin{align}\label{eq:normalization1}
\sum_{j=1}^m \gamma_j = 1.
\end{align}
For a meaningful dynamics a state of the population has to stay within the simplex $\gamma\in S_m$.

Evolutionary dynamics evolves when the individual fitness $f_j(\gamma)$ of type $E_j$ is compared to the average fitness of the population $\bar{f}(\gamma)=\sum \gamma_j f_j(\gamma)$ to yield the continuous replicator equation

\begin{align}\label{eq:replicator equation}
\dot{\gamma}_j=\gamma_j (f_j(\gamma)-\bar{f}(\gamma)).
\end{align}
Note that if a fitness $f_j(\gamma)$ is larger than the average, $\gamma_j$ will increase over time, meaning that the type $E_j$ becomes a larger share of the population and vice versa.

It is an important observation that for a starting $\gamma^{(0)}\in S_m$, any future $\gamma(t)$ derived from Eq.~\eqref{eq:replicator equation} will also stay within the simplex $S_m$~\footnote{This can be seen with $S(t)=\sum_j \gamma_j(t)$ and $\dot{S}=(1-S)\bar{f}$ which has $S(t)=1$ as a solution. It follows that any state starting with the right normalization Eq.\eqref{eq:normalization1} also stays on the plane $S(t)=1$ for any future $t>0$~\cite{HofbauerSigmund1998}.}. Furthermore, if a type $E_j$ is not in the initial state $\gamma_j(0)=0$, it will stay extinct for all $t>0$.

In EDT, we want to avoid any type $j$ \emph{becoming} extinct, i.e. $\gamma_j(t)=0$ for $t>0$. As we will see further below, a major effort in the stability of EDT is to prevent $\gamma_j(t)=0$ for any $j$, respectively to prove that this will not happen for any kind of input data.

For a system described by Eq.~\eqref{eq:replicator equation} the natural questions arise: under which conditions and initial starting points $\gamma^{(0)}$ will the dynamics lead to evolutionary stable states $\hat{\gamma}$ where $\dot{\gamma}_j=0$ for all $1\le j\le m$? Are there regions of $S_m$ where states are attracted to a certain $\hat{\gamma}$?

For systems described by the replicator equation Eq.~\eqref{eq:replicator equation} there holds the quotient rule
\begin{align}\label{eq:quotient rule}
\left(\frac{\gamma_j}{\gamma_\ell} \right)^{\cdot}=\frac{\gamma_j}{\gamma_\ell}\left(f_j(\gamma)-f_\ell(\gamma)\right)
\end{align}
which follows directly from the derivative quotient rule and a consequent redefinition of the time scale. As direct consequence of Eq.\eqref{eq:quotient rule} we see that the dynamics does not change if $f_j$ contains a constant independent of $j$. In particular, only relative, differential shares of fitness or resources determine the outcome of the evolutionary game.

In EGT, a major focus of research are linear systems where Eq.~\eqref{eq:replicator equation} simplifies to

\begin{align}\label{eq:linear replicator equation}
\gamma_j=\gamma_j \left([A\gamma]_j-\gamma\cdot A\gamma\right)
\end{align}
with the payoff matrix $A\in \mathbb{R}^{m\times m}$. For linear systems, theorems describing the conditions for evolutionary dynamics and stability are well known~\cite{MaynardSmith1981,HofbauerSigmund1998}.

\subsection{Equivalency to Lotka-Volterra equations}
\label{sec:lotkavolterraEGT}
One can show that, in the linear case Eq.~\eqref{eq:linear replicator equation}, the replicator dynamics is equivalent to the behavior of Lotka-Volterra~(LV) equations~\cite{HofbauerSigmund1998}
\begin{align}
  \dot{y}_j = y_j(b_j+\sum_{\ell=1}^{m-1}a^{\prime}_{j\ell}y_\ell),\label{eq:lotkavolterra1}
\end{align}

The proof uses the fact that the addition of a constant $c_j$ to the $j$-th column does not change the linear replicator equation Eq.~\eqref{eq:linear replicator equation}. Consequently, one can map the problem to an equivalent system where the last row of $A=(a_{j\ell})$ is zero~($a_{m\ell}=0$ for $1\le\ell\le m$).

The corresponding mapping in our case reads
\begin{align}
  b_j &= a_{jm}-a_{mm},\label{eq:LV1}\\%
  a_{j\ell}^\prime &= a_{j\ell}-a_{m\ell}.\label{eq:LV2}
\end{align}
Setting $y_m=1$ and using the coordinate transformation
\begin{align}
  \gamma_j = \frac{y_j}{\sum_{\ell=1}^m y_\ell}\label{eq:LV3}
\end{align}
we manage to reduce the problem with the quotient rule Eq.~\eqref{eq:quotient rule}
\begin{align}
  \dot{y}_j = \left(\frac{\gamma_j}{\gamma_m}\right)^.  =\left(\frac{\gamma_j}{\gamma_m}\right) \left[(A'\gamma)_j-(A'\gamma)_m\right]= y_j (A'\gamma)_j
\end{align}
where we used $(A'\gamma)_m=0$. Finally, with $y_m=1$ we arrive at equation Eq.~\eqref{eq:lotkavolterra1}.

This bijective mapping between linear replicator systems and Lotka-Volterra's method allows us to access fundamental theorems of LV. In particular, since Eq.~\eqref{eq:lotkavolterra1} for $\dot{y}=0$ is a linear system of equations, we can directly solve Eq.\eqref{eq:linear replicator equation} and find a fix point solution. Note that this only represents a valid solution if $A'y=-b$ leads to $\gamma\in S_n$.


\section{Formalism}
\label{sec:formalism}

\subsection{Problem definition}
We start our formulation of evolutionary data theory, by defining format of the input data. In general, we restrict ourselves to structured data which can be represented as a $n\times m$ matrix
\begin{align}
X = 
\begin{pmatrix} 
a_{11}      &       \cdots      &    a_{1m}     \\    
\vdots      &       \ddots      &    \vdots \\        
a_{n1}      &       \cdots      &    a_{nm}  \\
\end{pmatrix},
\end{align}
where $a_{ij}$ can be boolean, float, list, text, etc. format. In particular, $a_{ij}$ can have different format for different $j$ - the only important restriction is that the format does not change along a particular column $j$\footnote{In optimization problems, rows of $X$ are often called agents and columns features, respectively.} .

One central paradigm of evolutionary data theory is to interpret the rows and columns as organisms and genes, respectively. This mapping allows us to talk of organism and gene fitness when comparing rows and columns and consequently gives us access to the rich formalisms of EGT~(please see Wissgott\cite{wissgott2025} for a detailed derivation of this step).

Let us quickly summarize the next steps
\begin{enumerate}
    \item Convert the input data into an evolutionary probability matrix
    \item Define a replicator equation governing the dynamics of the evolutionary system
    \item Define evolutionary strategies describing how the evolutionary entities transfer fitness amongst each other.
    \item Derive conditions for convergence, permanence and uniqueness of solutions.
    \item Solve the system to obtain the fitness values for all evolutionary entities.
\end{enumerate}

As a first step and to make the columns~(genes) comparable we will normalize $X$ such that all entries fall into $[0,1]$. To that end, let us define a genetic fitness function $\phi_j$ for every particular gene.
\begin{align}
\Phi(X) = 
\begin{pmatrix} 
\phi_1(a_{11})      &       \cdots      &    \phi_m(a_{1m})     \\    
\vdots      &       \ddots      &    \vdots \\        
\phi_1(a_{n1})      &       \cdots      &    \phi_m(a_{nm})  \\
\end{pmatrix},
\label{eqn:genevariantfitnessmatrix}
\end{align}
with $\Phi\in \mathbb{R}_{[0,1]}^{n\times m}$ a probability matrix where larger values means better fitness. Note that we will assume that $\Phi$ will not contain any trivial columns. Additionally, we will merge all equal columns into one single column~(since they contain no new information about the system).

With $\phi_j(a_{ij})=\phi_{ij}$ we will use the following moment definitions for genes
\begin{align}
    \widetilde{\phi}_j&= \frac{1}{n}\sum_{i=1}^n \phi_{ij},\\
    \xi_{ijt\ell} &= \phi_{ij}\phi_{t\ell},\\
    \widetilde{\xi}_{j\ell}&= \frac{1}{n}\sum_{i=1}^n \xi_{iji\ell}
\end{align}

\subsection{Motivation}\label{sec:motivation}
Before delving deeper into the depth of mathematical derivation let us contemplate about the reason why EDT works in the first place. Why does an evolutionary perspective of structured matrix data return meaningful insights? 

\begin{description}
    \item[Independent optimization] Evolutionary systems which are disconnected from each other, sometimes lead to different but also sometimes to similar solutions~('convergent evolution'). In particular, the dynamics show no bias towards a particular predefined behavior. Independent evolutionary 'islands' all offer an individual optimization leading to the fittest entities for that particular setting. In our case, this means that all sets of input data lead to an individual, unbiased result.
    \item[Trade-off] Evolutionary systems shine by determining whether it pays to invest resources in a particular trait. Hence, evolutionary algorithms are well suited to balance a trade-off between mutual competing data features. Furthermore, many human and technical problems boil down to trade-offs between gains and costs.
    \item[Recursive Behavior] The dynamics of replicator equations show a rich diversity in behavior. This is rooted in the fact that these equation repeatably feedback the system onto itself via recursively applying an evolutionary function. Heuristically one can observe a deeper level of analysis as compared to pure linear approaches.
    \item[Adaptiveness] The initial, evolutionary setup can be tweaked to take into account certain known preferences in terms of data features, giving these traits an advantage in the evolutionary contest. Hence, evolutionary approaches are well suited for problems where individuality is more important than pure statistical behavior. 
\end{description}

\subsection{Replicator Equation for Evolutionary Data Theory}
\label{sec:replicatorequationsEDT}
Like the Schrödinger equation in quantum physics, the replicator equation
\begin{align}
  \gamma_j^{(k+1)} = \gamma_j^{(k)} \frac{1+h\Delta_j^{(k)}}{\sum_{\ell=1}^m \gamma_\ell \left(1+h\Delta_\ell^{(k)}\right)},\label{eqn:replicator1}\\
\end{align}
with $k$ denoting the iteration and the step size $0<h<1$ can be considered the 'backbone' of EDT. Note that the denominator ensures normalization and the delta functions $\Delta_j$ represent the share of differential fitness a certain gene obtains for that iteration.

In general Eq.\eqref{eqn:replicator1} can have (i) one single valid rest point in $\interior S_m$, (ii) multiple rest points in $\interior S_m$, depending on the initial $\gamma_j^{(0)}$, (iii) boundary solutions with some or all $\gamma_j\rightarrow 0$, (iv) oscillating behavior with controlled orbits, (v) chaotic orbits, (vi) no valid solution.

In this paper we will focus on the case of (i) where our objective is to find evolutionary strategies which yields a unique rest point independent of the initial starting point $\gamma^{(0)}$. In particular, the method should be applicable for any input matrix $\Phi$.

\subsection{Evolutionary Strategies}
\label{sec:evolutionarystrategies}
In EDT, evolutionary strategies play a slightly different role as in EGT. In EGT, evolutionary strategies represent a certain type of behavior gaining or losing fitness in the contest. In EDT, evolutionary strategies define the evaluation of the delta functions $\Delta_j$, depending on the input data $\Phi$ and the current fitness $\gamma$.

Evolutionary strategies in EDT are usually split into
\begin{align}
\label{eq:geneorgdelta}
		\Delta_j(\gamma)= \Delta_j^{g}(\gamma) + \Delta_j^{\omega}(\gamma)
\end{align}
where $\Delta_j^{g}(\gamma)$ and $\Delta_j^{\omega}(\gamma)$ denote the strategies for genes and organisms, respectively~\cite{wissgott2025}. By splitting the two contribution we get two 'perspectives' on the data problem. The gene contribution $\Delta_j^{g}(\gamma)$ takes into account the perspective of the data feature in the evolutionary contest, while the organism part $\Delta_j^{\omega}(\gamma)$ covers fitness comparisons between rows of the input matrix.

Whereas the genetic Delta function $\Delta_j^{g}(\gamma)$ can be interpreted as the function $f_j$ in evolutionary game theory, the organism perspective does not have a clear analogue in EGT. Both Delta functions are defined by the differential fitness contribution averaged over all organisms
\begin{align}
\label{eq:averagedgeneorgdelta}
		\Delta_j^{g}(\gamma)&= \frac{1}{n}\Delta_{ij}^{g}(\gamma),\\
        \Delta_j^{\omega}(\gamma)&=\frac{1}{n}\Delta_{ij}^{\omega}(\gamma).
\end{align}
The motivation behind this step is that $\Delta_{ij}^{g}(\gamma)$ takes into account the fitness gain or loss wrt. a particular gene expressed in a certain organism. Thus, $\Delta_{ij}^{g}(\gamma)$ usually contain dependencies on 'pure' gene quantities like $\gamma_j$. In contrast, the Delta contribution $\Delta_{ij}^{\omega}$ denotes the push or penalty for the $j$th gene seen from a particular organism. Here, we often find dependencies on the organism fitness $r_i$.

Let us continue by introducing the first set of evolutionary strategies, the dominant~($\operatorname*{dom}$) gene strategy and the balanced~($\operatorname*{bal}$) organism strategy
\begin{align}\label{eq:principledeltadef}
		\Delta_{ij}^{g:\operatorname*{dom}}(\gamma)&=\gamma_j\left(\phi_{ij}-\frac{1}{2}\right), \\\Delta_{ij}^{\omega:\operatorname*{bal}}(\gamma)&= -2\,r_i\left(\frac{\gamma_j\phi_{ij}}{r_i} - \frac{1}{m}\right).
	\end{align}
Whereas the definition of $\Delta_{ij}^{g:\operatorname*{dom}}(\gamma)$ is obviously a direct 'the-larger-the-better' functional dependence on the input data, the organism part $\Delta_{ij}^{g:\operatorname*{dom}}(\gamma)$ requires further explanation. Assuming a row of input data represents a evolutionary unit like an animal, depending on smaller, atomic units like genes, $\Delta_{ij}^{\omega:\operatorname*{bal}}(\gamma)$ measures how dependent an organism is wrt. to a certain gene. In particular, we observe that $\Delta_{ij}^{\omega:\operatorname*{bal}}(\gamma)$ reduces the fitness of genes that they are highly dependent of and vice versa~\cite{daniilidis2025, wissgott2025}.

In is an important observation that both contributions, $\Delta_{ij}^{g:\operatorname*{dom}}(\gamma)$ and $\Delta_{ij}^{\omega:\operatorname*{bal}}(\gamma)$, are mutually repulsive in the sense that a growth of one delta function leads to the reduction of the other. We will use this property to our advantage when we show the convergence to a unique rest point further below.

It can be shown that in the case of the DomBal combination of gene and organisms evolutionary strategies, the complete Delta function takes a very simple form~\cite{daniilidis2025}
\begin{align}
\label{eq:dombal}
		\Delta_j^{\operatorname*{dombal}}(\gamma)= -\gamma_j\Bigl(\widetilde{\Phi}_j+\tfrac12\Bigr)
	\;+\;\frac{2}{m}\sum_{s=1}^m\gamma_s\widetilde{\Phi}_s.
	\end{align}
Note that $\Delta_j(\gamma)$ contains only averages, respectively first moments $\widetilde{\Phi}_j$ of the columns of the input data $\Phi$. 

An evolutionary strategy taking into account the second moments of $\Phi$ is the combination of the altruistic~($\operatorname*{Alt}$) gene strategy with the selfish~($\operatorname*{Sel}$) organism strategy. To introduce these strategies, we need the concept of kinship~(similarity) between genes and organisms
\begin{align}
 \label{eqn:genekinship}
 \kappa_{j\ell}^g  = 1-\frac{\norm{\bm{g}_j-\bm{g}_\ell}}{n},\\
 \label{eqn:orgkinship}
 \kappa_{it}^\omega  = 1-\frac{\norm{\bm{\omega}_i-\bm{\omega}_t}}{m},
\end{align}
where we used $\bm{g}_j=[\phi_{1j},\ldots,\phi_{nj}]$ and $\bm{\omega}_i=[\phi_{i1},\ldots,\phi_{im}]$.

Additionally we will use the average difference between the column means
\begin{align}
\widetilde{\nu}^{g} &= \frac{1}{m^2}\sum_{j,\ell=1}^{m}|\widetilde{\phi}_j-\widetilde{\phi}_\ell|. \label{eq:nugdef}
\end{align}
the harmonic organism fitness
\begin{align}
r_i^h&= \frac{1}{m}\sum_j \phi_{j}(a_{ij})\label{eq:harmonic}
\end{align}
and the average difference between the harmonic organism fitness values
\begin{align}
\widetilde{\nu}^{\omega} &= \frac{1}{n^2}\sum_{i,t=1}^{n}|r_i^h-r_t^h|. \label{eq:nuorgdef}
\end{align}
Then, we define the altruistic gene strategy as
    \begin{align}
		\tilde{\Delta}_{ij}^{g:alt} &= \frac{1}{\,\widetilde{\nu}^g}\sum_{\ell, \ell\neq j}^m \gamma_\ell\kappa_{j\ell}^g \left[\phi_{\ell}(a_{i\ell})-\phi_{j}(a_{ij})\right], \label{eqn:gs-altruismtilde}\\
\Delta_{ij}^{g:alt} &= \Delta_{ij}^{g:dom}\cdot\tilde{\Delta}_{ij}^{g:alt}. \label{eqn:gs-altruism}
	\end{align}
Note that $\Delta_{ij}^{g:alt}$ takes into account kinship relationships between genes as well as the pure dominant behavior of $\Delta_{ij}^{g:dom}$. The motivation behind the naming of the altruistic strategy lies in the fact that genes can transfer fitness to related genes mimicking altruistic behavior~\cite{wissgott2025}.
    
Analogously, we define the selfish organism strategy as
\begin{align}
\tilde{\Delta}_{ij}^{\omega:sel} &= \frac{1}{n\,\widetilde{\nu}^{\omega}}\sum_{t, t\neq i}^m \kappa_{it}^\omega \left(r_i-r_t\right), \label{eqn:os-selfishtilde}\\
\Delta_{ij}^{\omega:sel} &= \Delta_{ij}^{\omega:bal}\cdot\tilde{\Delta}_{ij}^{\omega:sel}, \label{eqn:os-selfish}
\end{align}
where organisms tend to selfishly reduce the fitness of close relatives~\cite{wissgott2025}.

As we will see further below the combination
\begin{align}
\label{eq:altsel}
		\Delta_j^{\operatorname*{altsel}}(\gamma)= \Delta_j^{g:\operatorname*{alt}}(\gamma) + \Delta_j^{\omega:\operatorname*{alt}}(\gamma)
\end{align}
again shows mutually repulsive dynamics leading to the existence of a unique rest point.

\subsection{General rules for evolutionary strategies}
\label{sec:rulesforevolutionarystrategies}
In contrast to EGT, where one quite often investigates a system with a particular payoff matrix, there is in general no restriction to the input data $\Phi$ other than that it is accurately normalized. Hence, when applying EDT, one faces the difficulty that it has to return meaningful results for \emph{any} $\Phi$.

Consequently, instead of showing properties like convergence for a particular example, these properties have to be intrinsically proven for a particular combination of evolutionary strategies. To that end, we can define general rules that the function $\Delta_j(\gamma)$ and the connected replicator equation has to fulfill
\begin{description}
    \item[Convergence] there has to be a unique rest point. This property is pivotal for two reasons: first, one does not want the solution of a data problem to depend on the initial importance of features, which would introduce a certain ambiguity. Second, the existence of unique rest points can be used to speed up convergence by several orders of magnitudes\footnote{Note that in certain applications one \emph{wants} that the convergence depends on the initial $\gamma^{(0)}$, e.g. for problems where personalization plays an important role. In these examples, the convergence rule can be relaxed.}.
    \item[Persistence] all genes available in the initial population have to be persistent, i.e. they will not become extinct as $k \rightarrow \infty$. If genes would become extinct, the corresponding data feature will not be taken into account at all for the solution of our data problem.
    
\end{description}

\subsection{Bishop-Cannings Theorem for Evolutionary Data Theory}
\label{sec:bishopcannings}
In the early times of evolutionary game theory, the theorem first introduced by Bishop and Cannings proved a major steps forward in the analysis of biological systems~\cite{BISHOP197885}. The main reason roots in the fact that the theorem allows for algebraic methods to be applied in the formal and numerical investigation of the dynamics in these systems for the first time~\cite{MaynardSmith1981}. 

Thematically, the Bishop-Cannings theorem is directly connected to the differential nature of the replicator equation~\eqref{eq:replicator equation}. This is also reminiscent in the quotient rule~\eqref{eq:quotient rule} which essentially says that any fitness transfer between evolutionary types vanishes for evolutionary stable strategies. 

In evolutionary data theory, the corresponding theorem plays an equally pivotal role. Again, it allows to analyze a recursive problem by algebraic means. Let us thus formulate the theorem in the picture of the delta functions of EDT:

\begin{theorem}[Bishop-Cannings theorem]\label{th:bishopcannings}
At any rest point of Eq.~\ref{eqn:replicator1} there holds
\begin{align}\label{eq:bishopcanningstheorem}
\Delta_1(\gamma)=\ldots =\Delta_m(\gamma)=\sum_{\ell=1}^m \gamma_\ell \Delta_\ell(\gamma)
\end{align}
\end{theorem}
\begin{proof}
\begin{align}
  \gamma_j^{(k+1)} - \gamma_j^{(k)} = h\gamma_j^{(k)} \frac{\Delta_j^{(k)}-\sum_{\ell=1}^m \gamma_\ell \Delta_\ell^{(k)}}{\sum_{\ell=1}^m \gamma_\ell \left(1+h\Delta_\ell^{(k)}\right)},\label{eqn:replicator2}
\end{align}
For a rest point, there holds $\gamma_j^{(k+1)} = \gamma_j^{(k)}$ for all $\le j \le m$. Hence, the nominator in Eq.\eqref{eqn:replicator2} has to vanish which proves the theorem.
\end{proof}

\subsection{Payoff matrices for Evolutionary Data Theory}
\label{sec:payoffEDT}
In EGT, payoff matrix are used throughout, in particular for linear systems~\eqref{eq:linear replicator equation}. Stemming from game theory, an entry $(a_{j\ell})$ models the fitness gain or loss when a type $E_j$ meets a contestant of type $E_\ell$. Note that usually, the matrix $A$ is not symmetric as can be seen in the famous hawk-dove game\cite{MaynardSmith1981}.

In EDT and for linear systems, we can analogously model the dynamics with 
 \begin{align}
  \Delta_j^{(k)}=[A\gamma^{(k)}]_j,\label{eqn:payoff1}
\end{align}
where $A\in \mathbb{R}^{m\times m}$. In particular $A$ does not depend on the gene fitness values $\gamma^{(k)}$. In the case for the DomBal combination of evolutionary strategies~\eqref{eq:dombal}, the this becomes
\begin{align}
A = 
\begin{pmatrix} 
\left(\frac{2}{m}-1\right)\widetilde{\phi}_1-\frac{1}{2}  &         \cdots      &    \frac{2}{m}\widetilde{\phi}_m     \\    
\vdots      &       \ddots      &    \vdots \\        
\frac{2}{m}\widetilde{\phi}_1      &       \cdots      &    \left(\frac{2}{m}-1\right)\widetilde{\phi}_m -\frac{1}{2} \\
\end{pmatrix}.\label{eq:dombalmatrix}
\end{align}
For the AltSel combination of evolutionary strategies~\eqref{eq:altsel}, both contributions are non-linear in $\gamma$. However, we can nevertheless rewrite both parts in terms of a matrix $D\in \mathbb{R}^{m\times m}$. Starting with the gene contribution, we have from~Eq.\eqref{eqn:gs-altruismtilde}

\begin{align}
\Delta_{ij}^{g,alt} &= \gamma_j \sum_{\ell=1}^m \gamma_\ell d_{iji\ell}^g,
\end{align}
which simply puts all the factors not depending on $\gamma$ into $d_{iji\ell}^g$. Note that $d_{iji\ell}^g$ can be precomputed once at the beginning of the iterative analysis. At the same time, we have for the organism contributions~Eq.\eqref{eqn:os-selfish}
\begin{align}
\Delta_{ij}^{\omega,sel} &= \gamma_j \sum_{\ell=1}^m \gamma_\ell d_{iji\ell}^\omega,
\end{align}
where we used the fact that the organism fitness values are defined as $r_i=\sum_m \gamma_\ell \phi_{i\ell}$. Averaging, we end up with
\begin{align}
D^g_{j\ell} &= \frac{1}{n}\sum_{i=1}^n d_{iji\ell}^g \label{eq:asDg}\\
D^\omega_{j\ell} &= \frac{1}{n}\sum_{i=1}^n d_{iji\ell}^\omega \label{eq:asDw}\\
D &= D^g+D^\omega\label{eq:asD}
\end{align}
As a short remark note that $D^\omega_{j\ell}$ is a symmetric matrix for offdiagonal elements. Without the asymmetric gene contribution $D^g_{j\ell}$ there would thus be no dynamics since there would be no effective fitness transfer between genes.

As stated before, and analogously to the linear case, $D$ does only depend on the input data $\Phi$ and can be thus precomputed once. Inserting this into the replicator equation~\eqref{eqn:replicator2} we obtain
\begin{align}
\gamma_j^{(k+1)} - \gamma_j^{(k)} &=  h\gamma_j^{(k)}\frac{\gamma_j \left[D \gamma^{(k)}\right]_j-\sum_{\ell=1}^m \gamma_\ell^{(k)} \gamma_\ell^{(k)}\left[D\gamma^{(k)}\right]_\ell}{1+\sum_{\ell=1}^m \gamma_\ell^{(k)} \gamma_\ell^{(k)}\left[D\gamma^{(k)}\right]_\ell}
\end{align}
Switching to the continuous picture and redefining the time scale we get
\begin{align}
\dot{\gamma} &= \gamma\left(\gamma\, D \gamma-\sum_{\ell=1}^m \gamma_\ell \gamma_\ell\left[D\gamma\right]_\ell\right)\label{eq:altselcontsmooth}
\end{align}

\subsection{Dynamics and Evolutionary Stability}
\label{sec:dynamicsandstability}
Let us now turn our attention to the first rule for evolutionary strategies in EDT and prove convergence for both combinations DomBal and AltSel. 

\begin{theorem}\label{th:dombalconvergence}
    The DomBal system~\eqref{eq:dombal} converges to a unique rest point $\hat{\gamma}\in S_m$ with
     \begin{align}
  \hat{\gamma}_j &= \frac{1}{\widetilde{\phi}_j+\frac12} \left(\sum_{\ell=1}^{m}\frac{1}{\widetilde{\phi}_\ell+\frac12}\right)^{-1}
\end{align}   
for $1\le j \le m -1$.
\end{theorem}
\begin{proof}
This theorem has been first proved in Daniilidis \emph{et al}\cite{daniilidis2025}. Here, we show another proof motivated by the equivalency to Lotka-Volterra equations. With the matrix form \eqref{eq:dombalmatrix} we obtain with Eqs.~(\ref{eq:lotkavolterra1}-\ref{eq:LV2}) for the mapping to the equivalent Lotka-Volterra system 
\begin{align}
  b_j &= \widetilde{\phi}_m+\frac{1}{2},\\
  a_{jj}^\prime &= -\left(\widetilde{\phi}_j+\frac{1}{2}\right)
\end{align}
for $1\le j \le m -1$, where all non-zero entries of
\begin{align}
  A^\prime y = -b
\end{align}
are in the diagonal $a_{jj}$. This system has the solutions
\begin{align}
  y_j &= \frac{\widetilde{\phi}_m+\frac12}{\widetilde{\phi}_j+\frac12}
  \end{align}
  and with Eq.\eqref{eq:LV3} we obtain
  \begin{align}
  \gamma_j &= \frac{\widetilde{\phi}_m+\frac12}{\widetilde{\phi}_j+\frac12} \left(1+\sum_{\ell=1}^{m-1}\frac{\widetilde{\phi}_m+\frac12}{\widetilde{\phi}_\ell+\frac12}\right)^{-1}\\
    &= \frac{\widetilde{\phi}_m+\frac12}{\widetilde{\phi}_j+\frac12} \left(\frac{\widetilde{\phi}_m+\frac12}{\widetilde{\phi}_m+\frac12}+\sum_{\ell=1}^{m-1}\frac{\widetilde{\phi}_m+\frac12}{\widetilde{\phi}_\ell+\frac12}\right)^{-1}\\
   &=\frac{1}{\widetilde{\phi}_j+\frac12} \left(\sum_{\ell=1}^{m}\frac{1}{\widetilde{\phi}_\ell+\frac12}\right)^{-1}
\end{align} 
which completed the proof.
\end{proof}

\begin{theorem}\label{th:altselconvergence}
    The AltSel system \eqref{eq:altselcontsmooth} with $\rank D = m$ converges to one or multiple rest points $\hat{\gamma}\in S_m$.
\end{theorem}
\begin{proof}
The replicator for AltSel can be written in discrete form as 
\begin{align}
\gamma_j^{(k+1)}&=  h\gamma_j^{(k)}\frac{1+\gamma_j^{(k)} \left[D \gamma^{(k)}\right]_j}{1+\sum_{\ell=1}^m \gamma_\ell^{(k)} \gamma_\ell^{(k)}\left[D\gamma^{(k)}\right]_\ell}.\label{eq:convergencereplicator1}
\end{align}
Hence, any $\gamma_j^{(k)}\in S_m$ and normalization is guaranteed by definition of the replicator equation. The proof is further structured in two steps: first, we will show that no $\gamma_j\rightarrow 1$. Second, we will show that Eq.~\eqref{eq:convergencereplicator1} converges faster then the linear system with $A=D$.
\smallbreak \noindent
	\textbf{Step 1.}  Assume $\gamma_j^{(k)}\rightarrow 1$ for some $1\le j\le m$.\\
    When $\gamma_j\rightarrow 1$, the $\operatorname*{Alt}$ part in Eq.\eqref{eqn:gs-altruismtilde} vanishes $\Delta_{ij}^{g,alt}\rightarrow 0$, since all other $\gamma_{j'}\rightarrow 0$ for $j\neq j'$ by Eq.~\eqref{eq:normalization1}. Hence, the behavior of the $\operatorname*{sel}$ part $\Delta_{ij}^{\gamma,sel}$ dominates the behavior. Since the $j$-th gene is taking over the population, we have 
    \begin{align}
     \Delta_{j}^{\omega,sel,(k)} > \Delta_{j'}^{\omega,sel,(k)}\label{eq:convergencereplicator1b}
    \end{align}
    for almost all $k>k'$ for some $k'>0$. We now use the simplified notation
    \begin{align}
     \Delta_{j}^{sel} = -\frac{2}{\widetilde{\nu} ^\omega\, n}\sum_{i=1}^n\left(\gamma_j\phi_{ij}-\frac{r_i}{m}\right)\sum_{t=1}^n \kappa_{it}^\omega  (r_{i}-r_{t}).
    \end{align}
    Note that the term $r_i/m$ does not depend on $j$ and can thus be omitted, since it adds the same fitness to all genes. Hence, we simplify
    \begin{align}
     \Delta_{j}^{sel} = -\frac{2\gamma_j}{\widetilde{\nu} ^\omega\, n}\sum_{i,t=1}^n\phi_{ij} \kappa_{it}^\omega  (r_{i}-r_{t}).
    \end{align}
    Let us now define organism pairs $(i,t)\in\mathcal{P}_n$ where the order of the indices is defined as $r_i>r_t$.
    If we resort the organism sum we get
    \begin{align}
     \Delta_{j}^{sel} = -\frac{2\gamma_j}{\widetilde{\nu} ^\omega\, n}\sum_{(i,t)\in \mathcal{P}_n^\prime}^n\left(\phi_{ij} - \phi_{tj}\right)\cdot\left(r_{i}-r_{t}\right),\label{eq:convergencereplicator2}
    \end{align}
    where $\mathcal{P}_n^\prime$ denotes all unique pairs for $n$ organisms. If $\gamma_j\rightarrow 1$ it will dominate the organism fitness values $r_i, r_t$. In particular, $r_i-r_t>0$ implies $\phi_{ij}-\phi_{tj}>0$, since the $j$th entry determines the sign of all $r$. Consequently, all factors in the sum in Eq.~\eqref{eq:convergencereplicator2} are positive and hence 
    \begin{align}
     \Delta_{j}^{sel} < 0\label{eq:convergencereplicator3}.
    \end{align}
    For the other $j'\neq j$, we have that $\gamma_{j'}\rightarrow 0$ and thus
    \begin{align}
     \Delta_{j'}^{sel} \rightarrow 0\label{eq:convergencereplicator4}.
    \end{align}
    It follows that 
    \begin{align}
     \Delta_{j}^{sel} < \Delta_{j'}^{sel}\label{eq:convergencereplicator5}
    \end{align}
    which contradicts the initial assumption Eq.~\eqref{eq:convergencereplicator1b}.
\smallbreak \noindent
	\textbf{Step 2.}  Using the continuous time variant of the replicator equation, we have
    \begin{align}
     \lim_{t\rightarrow \infty}|\dot{\gamma}_j| &< \lim_{t\rightarrow \infty}\left|\gamma_j [D\gamma]_j - \sum_{\ell=1}^m \gamma_\ell \gamma_\ell [D\gamma]_\ell\right|,\label{eq:convergencereplicator6}
    \end{align} 
    where we used $\gamma_j<1$ from Step 1. Compared to the linear variant $D\gamma$ there is an additional factor $\gamma_j<1$ in the replicator equation \eqref{eq:convergencereplicator1}. Consequently, the AltSel dynamics is dominated by the linear variant
    \begin{align}
     \lim_{t\rightarrow \infty}\left|\gamma_j [D\gamma]_j - \sum_{\ell=1}^m \gamma_\ell \gamma_\ell [D\gamma]_\ell\right|&<\lim_{t\rightarrow \infty}\left|[D\gamma]_j - \sum_{\ell=1}^m \gamma_\ell [D\gamma]_\ell\right|.\label{eq:convergencereplicator7}
    \end{align} 
    From Sec.\ref{sec:lotkavolterraEGT} we know that the existence of a unique solution is equivalent to the existence of the solution of the linear system of equations defined by $D$ with non-zero RHS. Hence, since we have assumed $\rank D=n$, there exist a unique solution $y$ to  the right-hand side of Eq.~\eqref{eq:convergencereplicator7}. Note that the linear solution may lie outside of the simplex $y\notin S_m$. However, for our purposes convergence to any unique rest point $\hat{y}$ is sufficient. We obtain the result that the replicator equation \eqref{eq:convergencereplicator1} converges to one or more rest points depending on the initial $\gamma(0)$ which completes the proof.
\end{proof}

\subsection{Persistence and Extinction}
\label{sec:peristence}
Note that the persistence of the DomBal combination trivially follows from Theorem~\ref{th:dombalconvergence}. Let us thus turn our attention to the other combination of evolutionary strategies presented in this paper.
\begin{theorem}
    For an AltSel system \eqref{eq:altselcontsmooth} where $\rank D = m $ with initial values 
    \begin{align}
    0 < \gamma_j(0) < 1 \text{ for all } 1\le j\le m.
    \end{align}
    Then, there exists a unique rest point $\hat{\gamma}$ and no gene becomes extinct, i.e. $\gamma$ stays in the interior of $S_m$ with 
    \begin{align}
    0 < \gamma_j(t) < 1 \text{ for all } 1\le j\le m
    \end{align}
    for $t\rightarrow \infty$.
\end{theorem}
\begin{proof}
It follows from Theorem~\ref{th:altselconvergence} that Eq.\eqref{eq:altselcontsmooth} converges to some rest point $\hat{\gamma}\in S_m$. To prove this theorem we thus only have to show that $\hat{\gamma}\in \interior S_m$ and that $\hat{\gamma}$ is unique. It follows from the Bishop-Cannings theorem Eq.~\ref{eq:bishopcanningstheorem} that 
    \begin{align}
    \hat{\gamma}_j [D \hat{\gamma}]_j = c \text{ for all } 1\le j\le m.
    \end{align}
    Here, $c$ could be zero, but only if $[D \hat{\gamma}]_j=0$ for all $j$. Hence we have
    \begin{align}
    [D \hat{\gamma}]_j = \frac{c}{\hat{\gamma}_j} \text{ for all } 1\le j\le m. 
    \end{align}
    The dynamics $\hat{\gamma}_j=\gamma_j\rightarrow 0$ would mean that the right-hand side would become undefined which contradicts the Bishop-Cannings theorem. Consequently, all $\hat{\gamma}_j>0$ which completes the proof.
\end{proof}
%



\section{Applications and Examples}
\label{sec:applications and examples}
In its current form, Evolutionary Data Theory can be applied to any problem that can be represented in matrix form. In the following we will provide a small example where a supermarket chain obtains a delivery of bananas and investigate certain questions that may arise in this situation.

Our, entirely virtual, chain of supermarkets has 10 stores at various distances from the logistic center where the banana delivery is currently located. We will take into account the following properties of the stores
\begin{description}
    \item[distance~(km)] the distance that a supplier has to drive to reach the store.
    \item[store space~(m2)] the total selling space the store has.
    \item[storage space left (m2)] the current free storage space. 
    \item[monthly revenue~(MEuro)] the total monthly revenue of a store in million Euros.
    \item[bananas sold~(kg)] the amount of bananas sold in the last month.
    \item[A/C grade~(1-3)] the quality of the A/C in the storage facilities of the store, graded from 1-3 where larger means better.
    \item[flagship store (1/0)] whether the store is a flagship store of the chain.    
\end{description}
If the example was real, some properties would update on a regular basis~(storage space left, monthly revenue, bananas sold, A/C grade), whereas others remain static~(distance, flagship). In the following we will assume that we take a snapshot of the day the banana delivery arrives at the logistic center. These data $X$ are shown in Tab.\ref{tab:stores}
\begin{table}[tbh]
		\centering
		\caption{The example data with 10 stores. See text for details on the data features.}
		\label{tab:stores}
		\begin{tabular}{lccccccc}
			\toprule
			     & distance &space & st. left & revenue & sold & A/C  &flagship \\
			\midrule
			store A&20& 400& 80& 2& 100& 1& 0\\
            store B&20& 250& 110& 1.2& 90& 2& 0\\
            store C &20& 300& 70& 2.3& 115& 2& 0\\
            store D  &40& 250&	130& 1&	50&	1& 1\\
            store E&50& 400&	30& 1.4& 80& 3& 1\\
            store F&70& 300& 100& 2.4& 90& 2& 0\\
            store G&100& 250& 120& 1.3& 100& 2& 0\\
            store H&150& 350& 70& 2.7&	110& 1&	1\\
            store I&200& 650& 40& 0.8&	140& 2&	1\\
            store J&200& 500& 60& 2.3&	120& 3&	0\\
			\bottomrule
		\end{tabular}
	\end{table}
In our example $X$ we have $m=7$ genes~(properties) and $n=10$ organisms~(stores). To normalize $X$ we now need to define the local gene fitness functions $\phi_j$. Here, we will use two fitness functions
\begin{align}
    \phi_j^{direct}(x_{ij}) &= \frac{x_{ij}}{\max_{1\le i \le n} x_{ij}},\\
    \phi_j^{inverse}(x_{ij}) &= \frac{1-x_{ij}}{\max_{1\le i \le n} x_{ij}},
    \end{align}
where $\phi_j^{direct}$ covers properties where larger is better and $\phi_j^{inverse}$ will be applied to properties where smaller is better, respectively. In our case, only for 'distance', $j=1$, smaller is better and $\phi^{inverse}$ will be used and for all other properties $2\le j \le m$, $\phi^{direct}$ will normalize the property.

Applying the normalization to $X$ from Tab.\ref{tab:stores}, we arrive  with Eq.~\eqref{eqn:genevariantfitnessmatrix} at the following matrix
\begin{align}
\Phi(X) = 
\begin{pmatrix} 
 0.90& 0.62& 0.62& 0.74& 0.71& 0.33& 0.00\\
 0.90& 0.38& 0.85& 0.44& 0.64& 0.67& 0.00\\
 0.90& 0.46& 0.54& 0.85& 0.82& 0.67& 0.00\\
 0.80& 0.38& 1.00& 0.37& 0.36& 0.33& 1.00\\
 0.75& 0.62& 0.23& 0.52& 0.57& 1.00& 1.00\\
 0.65& 0.46& 0.77& 0.89& 0.64& 0.67& 0.00\\
 0.50& 0.38& 0.92& 0.48& 0.71& 0.67& 0.00\\
 0.25& 0.54& 0.54& 1.00& 0.79& 0.33& 1.00\\
 0.00& 1.00& 0.31& 0.30& 1.00& 0.67& 1.00\\
 0.00& 0.77& 0.46& 0.85& 0.86& 1.00& 0.00\\
\end{pmatrix},
\label{eq:PhiExample}
\end{align}
with the mean values
\begin{align}
    \widetilde{\phi} = [0.57,\; 0.56,\; 0.62,\; 0.64,\; 0.71,\; 0.63,\; 0.40].\label{eq:storemeans}
    \end{align}
In the following sections, we will investigate several problems, comparing the results of the combination DomBal$=DB$ with the strategies AltSel$=AS$.

\subsection{Optimization Problems}
\label{sec:optimization}
Since in EDT we aim to determine the fitness of genes and organisms, discrete multi-objective optimization problems are one of the most straightforward ways to apply the method. Note that discreteness in this sense means that the solutions space is not continuous, but there are a fixed set of solutions $X$ and the objective is to find the optimal one.

Hence, for our example Tab.\ref{tab:stores}, we first ask: what is the most relevant data feature~(gene) to prioritize stores in terms of the banana delivery? In Fig.\ref{fig:genefitness} we compare the iterative gene fitness values for the two combinations of evolutionary strategies discussed in this paper. Note that in all numerical experiments, we start the algorithm with a harmonic gene fitness~$\gamma^{(0)}=1/m$ for all $1\le j\le m$.
\begin{figure}[tb]
		\centering
		\includegraphics[width=0.5\textwidth]{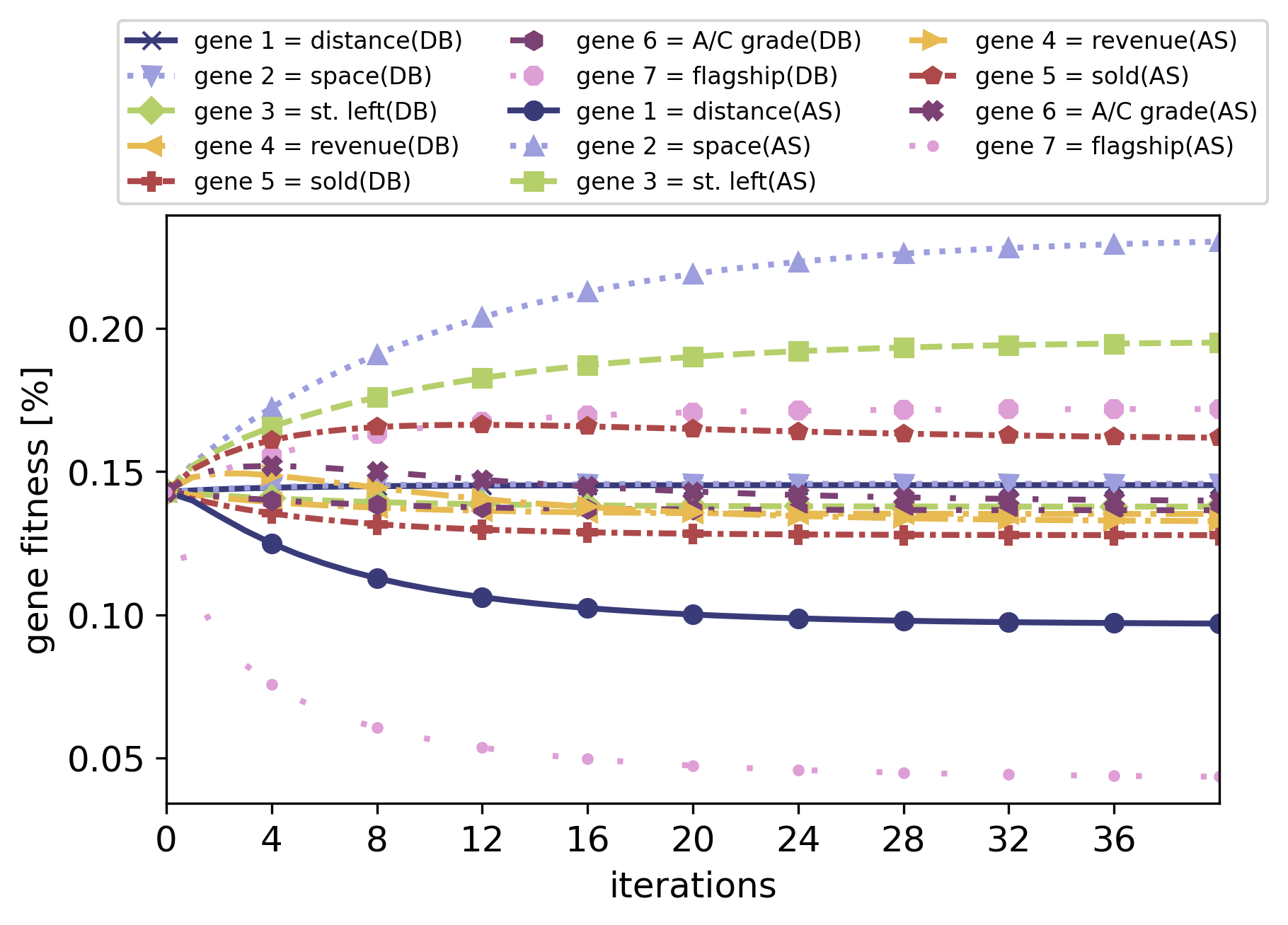}
		\caption{The gene fitness values for Tab.~\ref{tab:stores} over $40$ iteration comparing $DomBal=DB$ with $AltSel=AS$ evolutionary strategies~(we take harmonic initial values $\gamma^{(0)}_j=1/m=1/7$ for all genes $1\le j\le 7$).
		}
		\label{fig:genefitness}
	\end{figure}
    
It has been shown by Daniilidis \emph{et al}~\cite{daniilidis2025} that the unique rest point with the strategies DomBal depends only on the column averages Eq.\eqref{eq:storemeans}. In particular, the gene 'flagship' with the lowest average $\widetilde{\phi}_7=0.4$ is deemed the fittest. Hence, the rest point
\begin{align}
    \hat{\gamma} = [0.15,\, 0.15,\, 0.14,\, 0.14,\, 0.13,\, 0.14,\, 0.17]\label{eq:dbrestpoint}
    \end{align}
identifies the most relevant as $\hat{\gamma}_7=0.17$. However, since the mean values $\widetilde{\phi}$ are relatively close together, also the converged gene fitness~Eq.~\eqref{eq:dbrestpoint} do not differ much from the harmonic value. This shows a weakness of the simple combination DomBal that only depends on the first moments of the columns of $\Phi$.

For the AltSel combination of evolutionary strategies the result is quite different. Assembling the Alt part with~Eq\eqref{eq:asDg} we obtain
\begin{align}
D^{g} = 
\begin{pmatrix} 
 0.00& -1.24& -0.64& -0.91& -1.10& -0.15& -0.12\\
 -0.26& -0.81& 0.00& -0.52& -0.57& -0.68& -0.92\\
 -0.57& -0.59& -0.54& 0.00& -0.31& -0.49& -0.93\\
 -0.72& -0.31& -0.55& -0.27& 0.00& -0.29& -0.77\\
 -0.65& -0.42& -0.66& -0.45& -0.29& 0.00& -0.76\\
 -2.13& -1.72& -2.17& -2.16& -2.05& -2.05& 0.00
\end{pmatrix},
\label{eq:storesDg}
\end{align}
where we see that the in terms of dominance weakest gene 'flagship' $j=7$ altruistically shares fitness with all others~(the last row in Eq.\eqref{eq:storesDg}). Determining the Sel part with~Eq\eqref{eq:asDw}
\begin{align}
D^{\omega} = 
\begin{pmatrix} 
  -3.81& 1.46& -1.09& 0.19& 1.18& 0.69& 1.45\\
 1.46& -1.14& 1.15& 0.15& -0.69& -0.43& -0.91\\
 -1.09& 1.15& -1.94& 0.30& 0.75& 0.90& 1.32\\
 0.19& 0.15& 0.30& -1.85& -0.30& 0.11& 1.23\\
 1.18& -0.69& 0.75& -0.30& -0.89& -0.30& 0.43\\
 0.69& -0.43& 0.90& 0.11& -0.30& -1.79& 0.64\\
 1.45& -0.91& 1.32& 1.23& 0.43& 0.64& -7.69
\end{pmatrix},
\label{eq:storesDw}
\end{align}
showing the largest off diagonal fitness transfer at $D^{\omega}_{12}=D^{\omega}_{21}=1.46$ and $D^{\omega}_{17}=D^{\omega}_{71}=1.45$. Adding both parts together leads to the overall
\begin{align}
D = 
\begin{pmatrix} 
 -3.81& 0.22& -1.72& -0.72& 0.07& -0.34& 0.22\\
 0.84& -1.14& 0.62& -0.14& -0.73& -0.58& -1.04\\
 -1.35& 0.34& -1.94& -0.22& 0.18& 0.23& 0.41\\
 -0.37& -0.44& -0.24& -1.85& -0.61& -0.37& 0.30\\
 0.45& -1.01& 0.20& -0.57& -0.89& -0.58& -0.33\\
 0.04& -0.85& 0.25& -0.34& -0.58& -1.79& -0.12\\
 -0.68& -2.64& -0.84& -0.93& -1.62& -1.41& -7.69
\end{pmatrix},
\label{eq:storesD}
\end{align}
which indicates that the most successful gene is 'store space' having the most positive balance of fitness transfers between the genes, whereas 'flagship' has the worst\footnote{The diagonal entries of $D$ have no effect on the game dynamics, since only fitness transfers \emph{between} genes are relevant.}.

Consequently, the rest point of AltSel$=AS$
\begin{align}
    \hat{\gamma} = [0.09,\, 0.21,\, 0.19,\, 0.13,\, 0.16,\, 0.14,\, 0.07]\label{eq:asrestpoint}
    \end{align}
show that 'store space' $j=2$ is the most relevant data feature, while 'flagship' $j=7$ is by far the least important~(note that the gene fitness values in~Fig~\ref{fig:genefitness} are not fully converged after $40$ iterations).

    \begin{figure}[tb]
		\centering
		\includegraphics[width=0.5\textwidth]{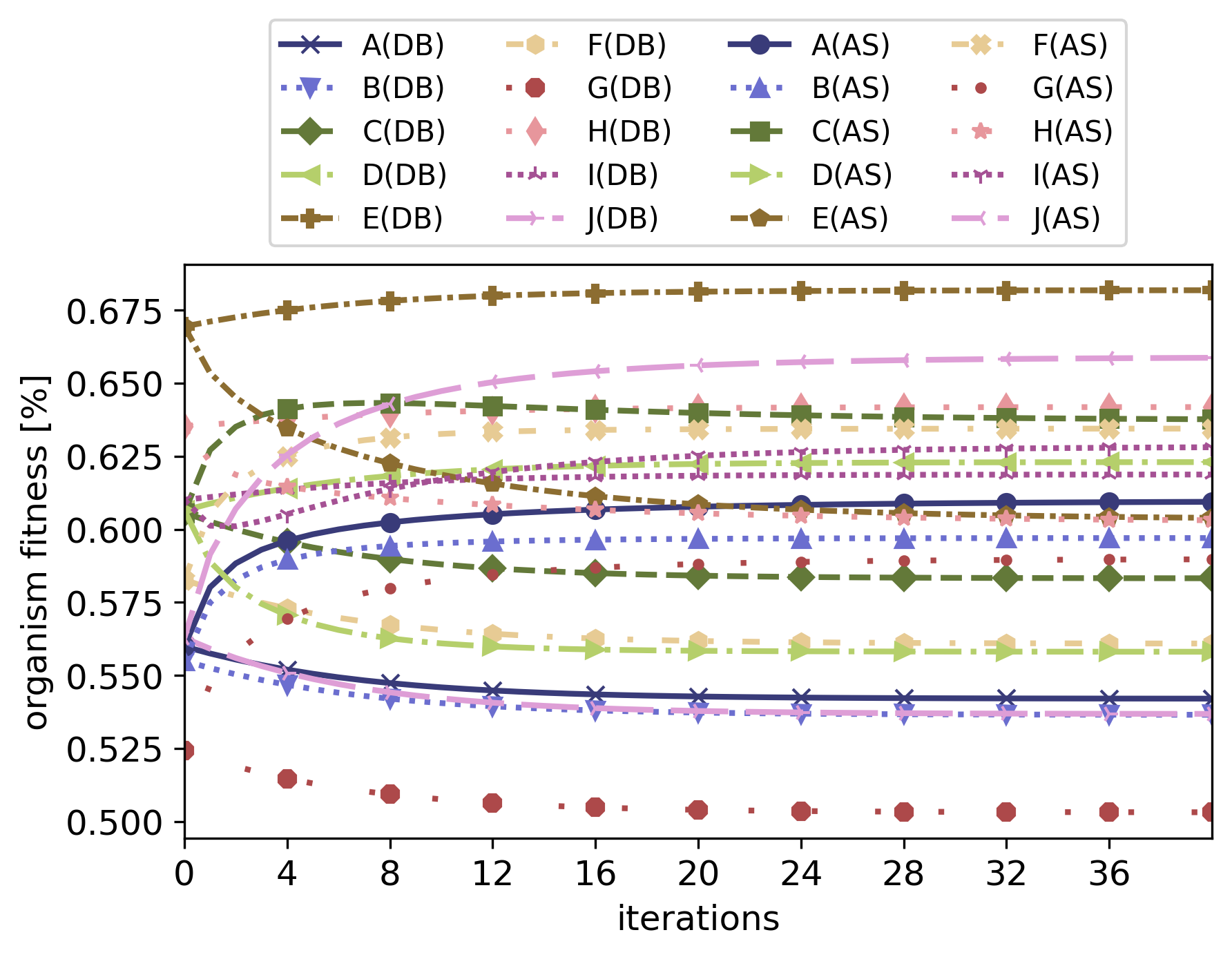}
		\caption{The organism fitness values for Tab.~\ref{tab:stores} over $40$ iteration comparing $DomBal=DB$ with $AltSel=AS$ evolutionary strategies. The organism/store with the highest fitness is thus the fittest player in the evolutionary game.
		}
		\label{fig:orgfitness}
	\end{figure}
As a next question we ask: which is the optimal store to supply first? In the evolutionary picture, this should clearly be the fittest organism~(store). In Fig.~\ref{fig:orgfitness} we thus show the iterative organism values comparing DomBal$=DS$ with AltSel$=AS$.

For DomBal, since the gene fitness values do only differ slightly from the harmonic case, the fittest store is E which has larger than average values for almost all properties and is also a flagship store.

For AltSel, the answer would be store J instead, which is relatively large with a good A/C. We observe that taking evolutionary strategies depending on the second moments leads to a richer dynamics and difference to the harmonic case when the first moments are so similar to each other.

\subsection{Distribution Problems}
\label{sec:distributionproblems}
Instead of asking ourselves which store to supply first, we may ask the more general question on how to quantitatively distribute the delivery. Here, it pays that EDT is more than a multi-objective sorting method, but returns fitness scores which can be used to plan the exact distribution.

Hence, taking the organism fitness values shown in~Fig.~\ref{fig:orgfitness}, we visualize in Fig.~\ref{fig:distribution} by how much we should deviate from the mean distribution rate a store normally gets of a delivery.
\begin{figure}[tb]
		\centering
		\includegraphics[width=0.5\textwidth]{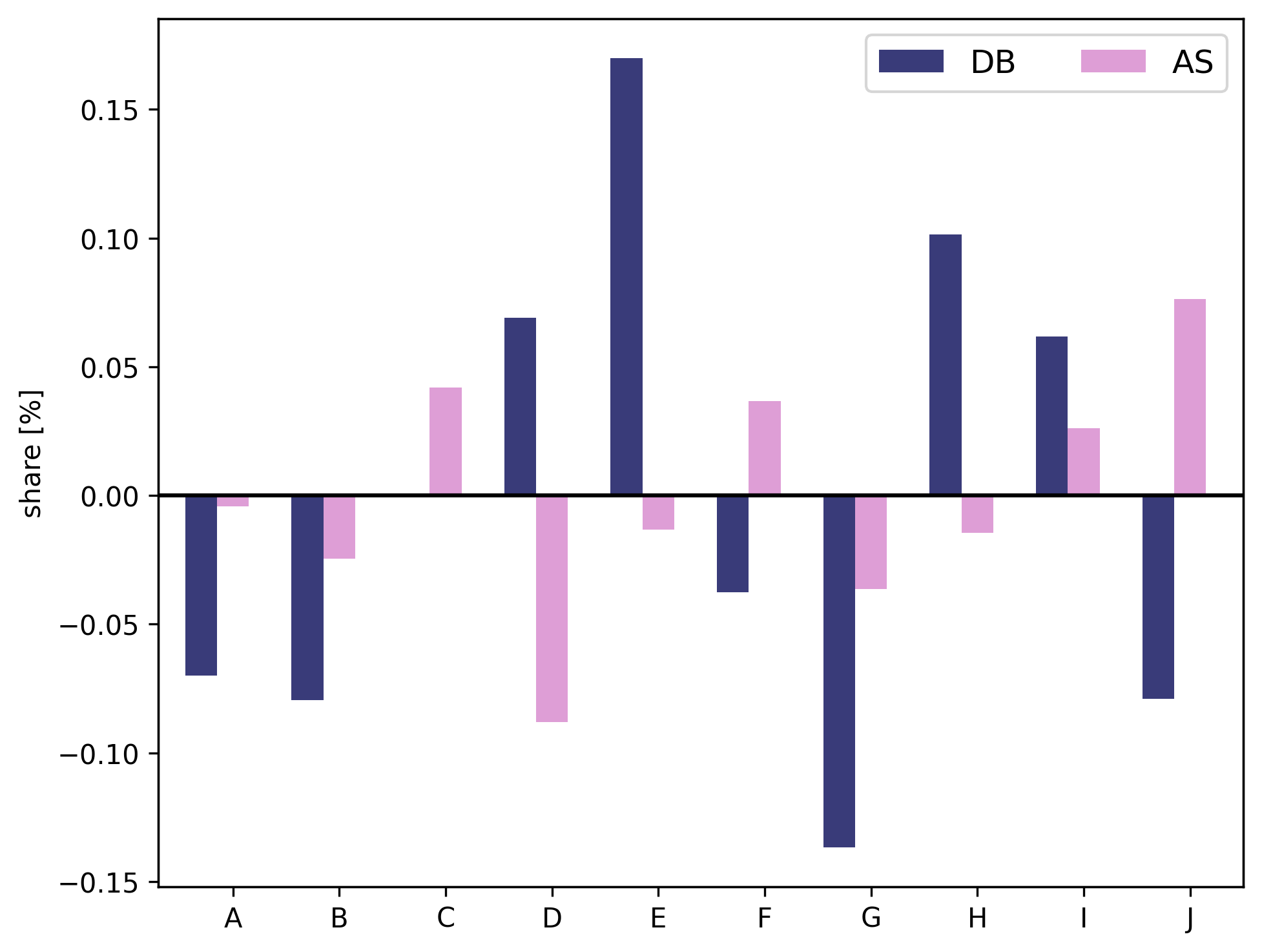}
		\caption{Relative deviation from mean distribution for $DomBal=DS$ and $AltSel=AS$ combinations of evolutionary strategies. The stores A-J are shown along the $x$-axis while the $y$-axis represents how much more or less should be distributed.
		}
		\label{fig:distribution}
	\end{figure}
For DomBal, we observe that the stores C, F, I and J should get more bananas wrt. to the average. Note that all of these stores have at least a grade $2$ A/C - therefore it pays to supply them first at a higher rate to reduce the storage time at the stores for the others.

For AltSel, we see that the stores D, E, H and I should get priority on the delivery. Here, the reasons for prioritization are more diverse: store D has the largest storage space left; store E is large and is a flagship store; store H has a high revenue; store I is selling a lot of bananas. In any case, EDT provides a straight-forward way to quantitatively approach distribution problems.

\subsection{Mixed Evolutionary Strategies }
\label{sec:mixedES}
So far, we choose either DomBal or AltSel as pure combinations of evolutionary strategies. We will now shortly introduce mixed strategies. To that end, we define the mixed gene delta function as
\begin{align}
\Delta_j^{g:mix}&=\alpha^{g:dom}\Delta_j^{g:dom}+\alpha^{g:alt}\Delta_j^{g:alt}\label{eq:mixedgenedelta}
\end{align}
and the mixed organism delta function as
\begin{align}
\Delta_j^{\omega:mix}=\alpha^{\omega:bal}\Delta_j^{\omega:bal}+\alpha^{\omega:sel}\Delta_j^{\omega:sel},\label{eq:mixedorgdelta}
\end{align}
where we consider the normalization $\sum_s \alpha^{g:s}=\sum_s \alpha^{\omega:s}=1$. 

As we have seen very diverse results with different ES, the selection of weights $\alpha$ represent a modeling choice for that particular system. In general, there are three ways to do that: (i) fixed choice depending on knowledge about the system or problem (ii) self-consistent determination and (iii) a machine learning approach. Method (ii) has been investigated heuristically in Wissgott~\cite{wissgott2025}. Let us turn our attention to determine the mixing through training.

\subsection{Machine Learning Applications}
\label{sec:machinelearning}
Assume we have the training data $\mathcal{X}=\{X_1,\ldots,X_q\}$ and target organism fitness scores $\mathcal{R}=\{R_1,\ldots, R_q\}$ available, where $R_\ell=\{r_1,\ldots, r_{n_\ell}\}$. Then, we can determine the optimal choice of $\alpha^g$ and $\alpha^\omega$ in Eq.~(\ref{eq:mixedgenedelta}+\ref{eq:mixedorgdelta}) by standard machine learning training minimizing the deviation of fitness values.

There are several ways to add additional degrees of freedom. Here, we introduce just one and leave the next steps to future work: instead of having one set of $\{\alpha^{g:dom}, \alpha^{g:alt}; \alpha^{\omega:bal}, \alpha^{\omega:sel}\}$ for the \emph{full} system, we might generalize that every gene $j$ obtains an \emph{individual} set of mixing weights $\{\alpha_j^{g:dom}, \alpha_j^{g:alt};\alpha_j^{\omega:bal}, \alpha_j^{\omega:sel}\}$. Note that convergence and persistence in this case is not yet proven, hence, these setups might experience a certain level of instability.



\section{Conclusion}
\label{sec:conclusion}
We have seen that Evolutionary Data Theory proves a versatile tool analyzing and solving general data problems. In particular, since we can be sure that EDT converges to a unique rest point with all data feature still present, our approach can be used in a wide range of applications.

The universality of the model of EDT becomes clear if we consider the opposite: assume we had to apply different rules and preconditioning for each particular set of input data. Such a method for data analysis would be instable and hard to maintain. In contrast, the proven theorems of convergence and persistence ensure that EDT can be applied out-of-the-box and return information about the relevance of data features, depending on the chosen set of evolutionary strategies.

While this paper shall introduce EDT to a wider audience, the theory has hardly come out of the cradle. With the convergence of self-consistent variants still to be proven and the derivation of completely new evolutionary strategies open, we just stand at the beginning of a rich, new field of mathematical research.



\bibliography{geneticai_v1}

\onecolumngrid

\end{document}